\documentclass[conference]{IEEEtran}
\IEEEoverridecommandlockouts
\usepackage{cite}
\usepackage{amsmath,amssymb,amsfonts}
\usepackage{bbm}
\usepackage{algorithm}
\usepackage{algpseudocode}
\usepackage{graphicx}
\usepackage{textcomp}
\usepackage{xcolor}
\usepackage{amsthm}
\usepackage{booktabs}
\usepackage{caption}
\usepackage{subcaption}
\usepackage{mathabx}

\newtheorem{theorem}{Theorem}
\newtheorem{lemma}{Lemma}
\newtheorem{remark}{Remark}

\newcommand{\red}[1]{\textcolor{red}{#1}}
\newcommand{\blue}[1]{\textcolor{blue}{#1}}
\newcommand{\orange}[1]{\textcolor{orange}{#1}}
\def\BibTeX{{\rm B\kern-.05em{\sc i\kern-.025em b}\kern-.08em
    T\kern-.1667em\lower.7ex\hbox{E}\kern-.125emX}}
\begin{document}

\title{Bayesian Learning in Episodic Zero-Sum Games
}

\author{Chang-Wei Yueh, Andy Zhao, Ashutosh Nayyar, and Rahul Jain
\thanks{The authors are with the Department of Electrical and Computer Engineering at the University of Southern California (email: \{cyueh, zhaozeha, ashutosn, rahul.jain\}@usc.edu).}}

\maketitle

\begin{abstract}
 We study Bayesian learning in episodic, finite-horizon  zero-sum Markov games with unknown transition and reward models. We investigate a posterior algorithm in which each player maintains a Bayesian posterior over the game model, independently samples a game model at the beginning of each episode, and computes an equilibrium policy  for the sampled model. We analyze two settings: (i) Both players use the posterior sampling algorithm, and (ii) Only one player uses posterior sampling while the opponent follows an arbitrary learning algorithm. In each setting, we provide guarantees on the expected regret  of the posterior sampling agent. Our notion of regret compares the expected total reward of the learning agent against the expected total reward under equilibrium policies of the true game. Our main theoretical result is an expected regret bound for the posterior sampling agent of order $\mathcal{O}(HS\sqrt{ABHK\log(SABHK)})$ where $K$ is the number of episodes, $H$ is the episode length, $S$ is the number of states, and $A,B$ are the action space sizes of the two players. Experiments in a grid-world predator--prey domain  illustrate the sublinear regret scaling and show that posterior sampling competes favorably with a fictitious-play baseline.


\end{abstract}

\begin{IEEEkeywords}
Reinforcement Learning, Game Theory
\end{IEEEkeywords}

\section{Introduction}
Markov games (also known as stochastic games \cite{shapley1953stochastic}) provide a fundamental framework for multi-agent systems by extending Markov decision processes (MDPs) to both competitive and cooperative multi-agent settings. In particular, two-player zero-sum Markov games model adversarial interactions where one agent’s gain is the other’s loss. If the game model is known to both players, Nash equilibrium strategies for a  zero-sum Markov game with a finite time horizon can be computed using a min-max dynamic program \cite{bacsar1998dynamic}. Our focus is on a multi-agent reinforcement learning  problem where  two players  play a zero-sum Markov game with unknown dynamics and rewards over multiple episodes of finite length.  

Single-agent reinforcement learning (RL) has been extensively studied in the prior literature and a variety of learning algorithms have been designed and analyzed \cite{jaksch2010near, osband2013more, gopalan2015thompson, ouyang2017learning, Azar17, Agrawal17}.  The learning problem becomes more challenging in the presence of  multiple agents that are learning independently. This is because in addition to learning the underlying game model, each agent also needs to figure out how best to respond to the other agent's policy that may be changing over time.   The best-response problem can be somewhat mitigated  by focusing on minimax policies that try to optimize the worst-case performance for an agent. For zero-sum games, minimax policies  are in fact Nash equilibrium policies and therefore learning a minimax policy for the true game effectively amounts to learning an equilibrium policy.

In this paper, we investigate  a posterior sampling (or Thompson sampling \cite{thompson1933likelihood}) based learning algorithm for agents in a two-player zero-sum Markov game. A posterior sampling based learning algorithm keeps track of the Bayesian posterior on the model of the Markov game. The algorithm periodically samples a model from this posterior distribution and plays an equilibrium (i.e., minimax) policy for the sampled model. We consider two cases - a) when both players use  posterior sampling algorithm with independent sampling and b) when one player uses the posterior sampling algorithm and the other uses an arbitrary learning algorithm.  In each setting, we provide guarantees on the expected regret  of the posterior sampling agent. Our notion of regret compares the expected total reward of the learning agent against the expected total reward under equilibrium policies of the true game.
 Our main theoretical result is an expected regret bound for the posterior sampling agent of order
\[
\mathcal{O}\left(HS\sqrt{ABHK\log(SABHK)}\right)
\]
where $K$ is the number of episodes, $H$ is the episode length, $S$ is the size of the state space, and $A$, $B$ are the sizes of the action spaces of the two players. This sublinear regret guarantee implies that as the number of episodes ($K$) grows,  the upper bound on regret-per-episode approaches zero.

\emph{Related Literature:}
A large body of work has focused on the exploration-exploitation tradeoffs in single-agent reinforcement learning. Algorithms based on the principle of {optimism in the face of uncertainty (OFU)}  \cite{AuCeFi02, jaksch2010near,Azar17}  as well as those based on  posterior sampling (PS) have been investigated~\cite{osband2013more, ouyang2017learning,gopalan2015thompson}. While OFU-based approaches involve construction of confidence sets for unknown system parameters and finding optimistic parameter values from these sets, PS-based approaches work with sampled parameter values drawn from the   posterior distribution on the unknown parameters. PS-based approaches balance exploration and exploitation by  periodically  sampling from the posterior distribution (exploration) and   then acting optimally with respect to the sampled parameter values (exploitation). PS-based approaches for single-agent RL are generally computationally simpler and achieve good empirical performance  \cite{osband2016why, ouyang2017learning,osband2013more}.


The problem of  finding equilibrium strategies in stochastic games with known dynamics and reward models has also received significant attention in the literature
 \cite{shapley1953stochastic}, \cite{Lago02}, \cite{Pero15,bacsar1998dynamic}.  More relevant for us is the literature on multi-agent reinforcement learning  in Markov games with unknown game models. One line of this work focuses on the offline setting where the learning procedures of different players are coordinated in order to find a Nash equilibrium \cite{sidford20a}, \cite{Zhang20},\cite{Bai20},\cite{Liu21},\cite{wei21a},\cite{chen22d, Jin22, Xie20a}. In online settings, on the other hand, players must learn independently, and the focus is on minimizing the regret with respect to the Nash equilibrium value \cite{Wei17}, \cite{jahromi2024bayesian}, \cite{Xie20a}, \cite{Tian21b}, \cite{Jin22}. In particular, \cite{Wei17} and \cite{jahromi2024bayesian} considered an infinite-horizon Markov game and analyze regret under a finite diameter  assumption about the Markov game, whereas our work  deals with finite horizon episodic games with no diameter or ergodicity-style assumption.  \cite{Tian21b} considers a setting where the opponent's action is not observable, which is a weaker information requirement than in our setting and therefore has an higher order term in regret (depends on $K^{2/3}$ while ours depends on $\sqrt{K}$). 
 The algorithms in  \cite{Jin22, Xie20a} are based on the OFU principle while we adopt a posterior sampling approach. OFU-based algorithms are computationally more demanding as they require a subroutine to find the optimistic parameters/value functions within a confidence set.  Further, the reward model in \cite{Xie20a} is a linear function of a known feature map, whereas our model allows for stochastic rewards with unknown distributions. Finally, \cite{Xie20a} gives a high probability regret bound that depends on $\sqrt K$ and $H^2$, whereas we have an expected regret bound that depends on $\sqrt K$ and $H^{1.5}$.



\emph{Notation:} For a set $X$,  $\Delta_X$ denotes the set of all probability distributions on $X$. For a positive integer $H$, $[H]$ denotes the set $\{1,\cdots,H\}$. $\mathbb{R}$ is the set of real numbers. $a \sim p$ indicates that $a$ is randomly generated according to the probability distribution $p$. $\lceil x \rceil$ denotes the smallest integer greater than or equal to $x$.
\section{System Model}
We consider a   two-player zero-sum Markov game  $M = (\mathcal{S}, \mathcal{A}, \mathcal{B},   \theta, R, H, \rho)$, where $\mathcal{S}$ is the finite state space, $\mathcal{A}$ and $\mathcal{B}$ are  finite action spaces for player 1 and player 2 respectively,  $H $ is the finite time horizon, and $\rho$ is the probability distribution of initial state. $\theta$ denotes the transition kernel with $\theta (s'|s,a,b)$ being the probability of transitioning  to state $s'$ from current state $s$ when actions $a$ and $b$ are selected by the two players. $R$ denotes the reward model  with $R(s,a,b)$ being the  probability distribution  of player 1's reward when the current state is $s$ and actions $a$ and $b$ are selected by the two players. We assume that the support of $R(s,a,b)$ is $[-1,1]$.  We use $\overline{R}^M(s,a,b)$ to denote the expected value of the distribution $R(s,a,b)$ of Markov game $M$. 

 We will consider the setting where the transition kernel $\theta$ and the reward model $R$ are unknown to both players but the rest of the Markov game model is known to both players. We will henceforth identify the Markov game $M$ by its transition kernel $\theta$ and the reward model $R$, i.e., $M = (\theta, R)$.

We consider an episodic game setting where the two players play the Markov game $M$ over multiple episodes with each episode having $H$ time steps.  The $k$th episode begins at time $t_k = (k-1)H + 1, $ for $k=1,2,\ldots$. At the beginning of the $k$th episode, the initial state is chosen according to the probability distribution $\rho$. At each time $t$, the following sequence of events occurs: (i) both players observe the current state $s_t$ and the previous actions $a_{t-1},b_{t-1}$, (ii) player 1  selects an action $a_t\in\mathcal{A}$ and player 2 selects an action $b_t\in\mathcal{B}$ simultaneously, (iii) player 1 obtains a  reward $r_t \sim R(s_t,a_t,b_t)$ and player 2 obtains a reward equal to  $-r_t$, (iv) the state transitions to $s_{t+1}\sim \theta(\cdot|s_t,a_t,b_t)$. The goal for player 1  is to maximize the total expected reward (hence we will refer to it as the maximizing player), while player 2's goal is to minimize the total expected reward (we will call it the minimizing player). A policy for player 1 is a function $ \mu:\ \mathcal{S}\times[H]\to\Delta_{\mathcal{A}}$. If player 1 is using  the  policy $\mu$ in the $k$th episode, then $a_t \sim \mu(s_t, t - t_k+1)$ for $t_k \le t < t_{k+1}$.
Similarly, a policy for player 2 is a function $ \nu:\ \mathcal{S}\times[H]\to\Delta_{\mathcal{B}}$. 

\subsection{Dynamic game preliminaries}
Given a Markov game $M$ and the policies $\mu$  and  $\nu$  for the two players, we define value function at step $h$ to be
\[
V^M_{\mu,\nu,h}(s)=\mathbb{E}^{\theta}_{\mu,\nu}\left[\sum_{l=h}^H \overline{R}^M(s_l,a_l,b_l) \Big|s_h=s\right],
\]
 where the expectation is  with respect to the probability distribution on state and action trajectories induced by the  polices of two players and the transition kernel. We also define the total expected reward for policies $\mu$ and $\nu$ in Markov game $M$ as follows:  
 \begin{align}
     J^M_{\mu,\nu}&
     =\sum_{s \in \mathcal{S}}\rho(s)V^M_{\mu,\nu,1}(s).
 \end{align}
Further, for the Markov game $M$ and  players' policies $\mu$ and $\nu$, we define the Bellman operator at step $h$  as follows: for any function $V: \mathcal{S} \to \mathbb{R}$, we have
\begin{align}
&\mathcal{T}^M_{\mu,\nu,h}V(s)= \notag \\
&\mathbb{E}_{a \sim \mu(s,h),b \sim \nu(s,h)}\left[\overline{R}^M(s,a,b)
+\sum_{s'\in\mathcal{S}}\theta(s'|s, a,b)V(s')\right]. \label{eq:bellman}
\end{align}

\begin{lemma} (Dynamic programming equation)
\[
V^{M}_{\mu,\nu,h}=\mathcal{T}^{M}_{\mu,\nu,h}V^{M}_{\mu,\nu,h+1}, 
\]
where $ V^{M}_{\mu,\nu,H+1}(s)=0\quad \forall s\in\mathcal{S}$.
\end{lemma}
\begin{proof}
    The result follows from  standard dynamic programming arguments \cite{kumar2015stochastic}.
\end{proof}

A policy pair $(\mu^{eq},\nu^{eq})$ is a Nash equilibrium in game $M$ if for all $\mu,\nu$, 
\begin{equation}
    J^M_{\mu,\nu^{eq}}\leq J^M_{\mu^{eq},\nu^{eq}}\leq J^M_{\mu^{eq},\nu}. 
\end{equation}
Given a Markov game $M$, the corresponding equilibrium policies can be obtained by a max-min dynamic program as described in the lemma below.
\begin{lemma} \label{lemma:minmaxdp}
For the Markov game $M$, define equilibrium value functions $V^{M}_{eq,h}$ backward inductively as follows:\\  $V^{M}_{eq,H+1}(s)=0$ and for $h = H, H-1, \ldots, 1,$
\begin{align}\label{eq:mindp}
    V^{M}_{eq,h}(s)=\max_{p \in \Delta_{\mathcal{A}}} \min_{q \in \Delta_{\mathcal{B}}} &\mathbb{E}_{a \sim p,b \sim q} \Big[ \overline{R}^M(s, a, b)\notag \\
    &+ \sum_{s'} \theta(s'| s, a, b) V^{M}_{eq,h+1}(s') \Big]\notag \\
    = \min_{q \in \Delta_{\mathcal{B}}}\max_{p \in \Delta_{\mathcal{A}}} &\mathbb{E}_{a \sim p,b \sim q} \Big[ \overline{R}^M(s, a, b)\notag \\
    &+ \sum_{s'} \theta(s'| s, a, b) V^{M}_{eq,h+1}(s') \Big].
\end{align}
        An equilibrium policy pair  $(\mu^{eq},\nu^{eq})$ is given by:
\begin{align}\label{eq:maxdp}
    \mu^{eq}(s,h) \in \arg\max_{p \in \Delta_{\mathcal{A}}} \min_{q \in \Delta_{\mathcal{B}}} &\mathbb{E}_{a \sim p,b \sim q} \Big[ \overline{R}^M(s, a, b) \notag \\
     &+ \sum_{s'} \theta(s'| s, a, b) V^{M}_{eq,h+1}(s') \Big],
\end{align}
\begin{align}\label{eq:mindp}
    \nu^{eq}(s,h) \in \arg\min_{q \in \Delta_{\mathcal{B}}}\max_{p \in \Delta_{\mathcal{A}}}  &\mathbb{E}_{a \sim p,b \sim q} \Big[ \overline{R}^M(s, a, b) \notag \\
     &+ \sum_{s'} \theta(s'| s, a, b) V^{M}_{eq,h+1}(s') \Big],
\end{align}
\end{lemma}
The proof of Lemma \ref{lemma:minmaxdp} is based on arguments similar to those used in \cite[Chapter 6]{bacsar1998dynamic} for a discrete-time, finite horizon game.\\
\emph{Note: We will use the notation $DP(M)$ to denote a pair of equilibrium policies obtained using the dynamic program of Lemma \ref{lemma:minmaxdp} for the Markov game $M$.}
\subsection{Learning Algorithms and Regret definition}\label{sec:regret_def}
Let $h_t = (s_1, a_1, b_1,r_1\cdots , s_{t-1}, a_{t-1},b_{t-1},r_{t-1})$ denote the state, actions and reward  history   before time $t$. We assume that both players know $h_t$ at time $t$. A learning algorithm for a player $i$ ($i=1,2$)  is a sequence of mappings $\psi^i_k, k=1,2,\ldots$ where, for each $k$, $\psi^i_k$ takes the history $h_{t_k}$ as input and (possibly randomly) outputs a policy for player $i$ to use in the $k$th episode. 

Let $\mu_k$ and $\nu_k$ denote the policies used by player 1 and player 2, respectively, in the $k$-th episode. 
Let $M^* = (\theta^*, R^*)$ denote the true Markov game and let $(\mu^*,\nu^*)$ be an equilibrium policy pair for the true Markov game $M^*$. Define 
\begin{equation}\label{eq:delta0}
    \Delta_k:=J^{M^*}_{\mu^*,\nu^*}-J^{M^*}_{\mu_k,\nu_k}.
\end{equation}
$\Delta_k$ is the difference between the  expected total reward of the equilibrium policies for the true game and the expected total reward of the  the actual policies used in episode $k$.  
We can now define player 1's regret over $K$ episodes as follows:

\begin{equation}\label{eq:regret_def}
    \mbox{Regret}(K)=\sum_{k=1}^K\Delta_k.
\end{equation}
\begin{remark}
    While we have defined the regret from player 1's perspective, it is easy to see that the zero-sum nature of the game implies that player 2's regret is just negative of player 1's regret. 
\end{remark}
\begin{remark}
    Each step of the dynamic program in Lemma \ref{lemma:minmaxdp} is a minmax optimization problem of a bilinear function. Such problems can be cast as a linear program and solved efficiently \cite[Chapter 2]{bacsar1998dynamic}. In our experiments, we used Clarabel \cite{Clarabel_2024} for solving these linear programs.
\end{remark}

\subsection{Bayesian Framework} \label{sec:Bayes}
We will adopt a Bayesian framework for the true Markov game as described below: 
\begin{enumerate}
    \item $\theta^*$ is a random matrix.
    \item The reward distribution for each state-action tuple comes from a parametrized family of distributions with support in $[-1,1]$.  To be precise, 
let $\mathcal{D} = \{d_{\lambda}: \lambda \in \mathbb{R}^n \}$ be  a parametrized  collection of probability distributions on the real line with support in $[-1,1]$. 
We assume that for each state-action tuple $(s,a,b)$, the reward distribution $R(s,a,b)$ belong to $\mathcal{D}$, i.e. $R(s,a,b) = d_{\lambda^*(s,a,b)}$ for some parameter $\lambda^*(s,a,b)$. Let $\lambda^*$ be the vector consisting of $\lambda^*(s,a,b)$ for all state-action tuples. We assume that the true $\lambda^*$ is a random vector. 
\item  $f_1$ be the joint prior distribution of $\theta^*$ and $\lambda^*$. For brevity, we will refer to the pair $\theta^*,\lambda^*$ as the MDP $M^*$.

\end{enumerate}

Our focus will be on $\mathbb{E}[\mbox{Regret}(K)]$ where the expectation is with respect to the prior distribution on $\theta^*,\lambda^*$ and the distribution  of the policies selected by the players' learning algorithms.


\section{Posterior sampling algorithm}
 We first consider the case where both players use a posterior sampling algorithm as their learning algorithm. This algorithm proceeds as follows for player $i$ ($i=1,2$): the player keeps track of a posterior distribution on $\theta^*, \lambda^*$ based on the observed state-action trajectory. Let $f_k$ denote the player's posterior distribution on $\theta^*, \lambda^*$  at the start of the $k-$th episode. The posterior distribution is updated according to Bayes' rule:  
\begin{align}\label{eq:posterior}
    &f_{k+1}(d\theta, d\lambda)\propto  \notag \\
    &\prod_{t=t_k}^{t_k+h-2} \theta(s_{t+1}|s_{t},a_{t},b_{t})\prod_{t=t_k}^{t_k+h-1}d_{\lambda(s_{t},a_{t},b_{t})}(r_{t}) f_k(d\theta, d\lambda)
\end{align}
Note that both players maintain the same posterior distribution since both have access to the same state-action-reward history.

At the start of the $k$th episode, player $1$ (respectively player $2$) draws a sample $M^1_k=(\theta^1_k,\lambda^1_k)$ (respectively $M^2_k=(\theta^2_k,\lambda^2_k)$) from the posterior distribution $f_k$. The players draw their samples independently of each other. Each player uses its sample to compute an equilibrium policy pair according to the dynamic program of Lemma \ref{lemma:minmaxdp}. That is, player 1 computes
\begin{equation}
    (\mu_k,\tilde{\nu}_k) = DP(M^1_k)
\end{equation}
and uses the policy $\mu_k$ in the $k$th episode while player 2 computes \begin{equation}
    (\tilde{\mu}_k,\nu_k) = DP(M^2_k)
\end{equation}
and uses the policy $\nu_k$ in the $k$th episode. The players' algorithms are summarized below.

\begin{algorithm}[H] 
\caption{Maximizer's (Player 1's) Algorithm}
\label{alg:maxalg}
\begin{algorithmic}[1]
\State Initialize prior distribution $f_1$.

\For {each episode $k = 1, 2, \dots, K$}
\State Sample $M^1_k \sim f_k$.
    \State Compute $(\mu_k,\tilde{\nu}_k) = DP(M^1_k)$ according to Lemma \ref{lemma:minmaxdp}.
    
    \For {each timestep $h = 1, 2, \dots, H$}
        \State Observe $s_h$  and sample action $a_h \sim \mu_k(s_h,h)$. 
        \State Observe $a_h,b_h,r_h$.
  
    \EndFor
    \State Update posterior distribution with the history according to \eqref{eq:posterior}.
\EndFor
\end{algorithmic}
\end{algorithm}

\begin{algorithm}[H]
\caption{Minimizer's (Player 2's) Algorithm}
\begin{algorithmic}[1]
\State Initialize prior distribution $f_1$.

\For {each episode $k = 1, 2, \dots, K$}
   \State Sample $M^2_k \sim f_k$.
    \State Compute $(\tilde{\mu}_k, \nu_k) = DP(M^2_k)$ according to Lemma \ref{lemma:minmaxdp}.
    
    \For {each timestep $h = 1, 2, \dots, H$}
        \State Observe $s_h$  and sample action $b_h \sim \nu_k(s_h,h)$.
        \State Observe $a_h,b_h,r_h$.
  
    \EndFor
    \State Update posterior distribution with the history according to \eqref{eq:posterior}.
\EndFor
\end{algorithmic}
\end{algorithm}

We can now state our main theoretical results.
\begin{theorem} \label{thm:main}
If both players use the posterior sampling algorithm, then
\[
|\mathbb{E}\left[\text{Regret}(K)\right]|\leq37HS\sqrt{AB KH \log(S AB KH)}.
\]
\end{theorem}

\begin{theorem}\label{thm:arb_opp}
If player 1 uses the posterior sampling algorithm (Algorithm \ref{alg:maxalg}), then regardless of the learning algorithm used by player 2, we have 
 \[
 \mathbb{E}[\mbox{Regret}(K)]\leq37HS\sqrt{AB KH \log(S AB KH)}.
 \]
 \end{theorem}

\section{Analysis}\label{sec:analysis}


The following lemma describes a key property of the posterior sampling algorithm.
\begin{lemma}\label{lemma:PS} (Posterior Sampling).
For any bounded function $g$ of Markov game and history $h_{t_k}$, 
\begin{equation}
    \mathbb{E}[g(M^*, h_{t_k})]=\mathbb{E}[g(M^1_k, h_{t_k})]=\mathbb{E}[g(M^2_k, h_{t_k})].
\end{equation}
\end{lemma}
\begin{proof}
    The lemma follows from  results in \cite{russo2014learning,osband2013more,ouyang2017learning}.
\end{proof}

To analyze the regret, we define two quantities related to $\Delta_k$ defined in \eqref{eq:delta0}.
\begin{equation}
\begin{aligned}
\hat{\Delta}^1_k:=J^{M_k^1}_{\mu_k,\tilde{\nu}_k}-J^{M^*}_{\mu_k,\nu_k},\\
\hat{\Delta}^2_k:=J^{M_k^2}_{\tilde{\mu}_k,\nu_k}-J^{M^*}_{{\mu_k,\nu_k}}.
\end{aligned}
\end{equation}
 In the definition of $\hat{\Delta}^1_k$, the second term (i.e. $J^{M^*}_{\mu_k,\nu_k}$) is the total expected reward under the policies used by the two agents in episode $k$ with the Markov game being $M^*$; the first term (i.e. $J^{M^1_k}_{\mu_k,\tilde{\nu}_k}$) is the total expected reward of the \emph{equilibrium policies} for the game $M^1_k$ sampled by player 1 in episode $k$.  A similar interpretation holds for $\hat{\Delta}^2_k$. The following lemma is a consequence of Lemma \ref{lemma:PS}.
\begin{lemma}  \label{lemma:regreteq}
\[
\mathbb{E}\left[\hat{\Delta}^2_k\right]=\mathbb{E}\left[ \Delta_k\right] =\mathbb{E}\left[\hat{\Delta}^1_k\right].
\]
\end{lemma}
\begin{proof}
\begin{align}
   &\mathbb{E}\left[\hat{\Delta}^1_k\right]= \mathbb{E}[J^{M_k^1}_{\mu_k,\tilde{\nu}_k}]-\mathbb{E}[J^{M^*}_{\mu_k,\nu_k}] \notag \\
   & = \mathbb{E}[J^{M_k^1}_{DP(M^1_k)}]-\mathbb{E}[J^{M^*}_{\mu_k,\nu_k}] \notag \\
   & = \mathbb{E}[J^{M^*}_{DP(M^*)}]-\mathbb{E}[J^{M^*}_{\mu_k,\nu_k}] \label{eq:r1} 
   =\mathbb{E}\left[{\Delta}_k\right],
\end{align}
where we used Lemma \ref{lemma:PS} in \eqref{eq:r1}. A similar argument applies to $\mathbb{E}\left[\hat{\Delta}^2_k\right]$.
\end{proof}
Next, we define $\tilde{\Delta}^1_k$ and $\tilde{\Delta}^2_k$ as
\begin{equation} \label{eq:hats}
\begin{aligned}
\tilde{\Delta}^1_k:=J^{M_k^1}_{\mu_k,\nu_k}-J^{M^*}_{\mu_k,\nu_k},\\
\tilde{\Delta}^2_k:=J^{M_k^2}_{\mu_k,\nu_k}-J^{M^*}_{{\mu_k,\nu_k}}.
\end{aligned}
\end{equation}
$\tilde{\Delta}^1_k$ is the difference between total expected rewards of policies $\mu_k,\nu_k$ under player 1's sampled Markov game $M^1_k$ and the true game $M^*$; similar interpretation holds for $\tilde{\Delta}^2_k$. We have the following result. 
\begin{lemma}\label{lemma:regret2}
  \begin{equation}
\mathbb{E}\left[\tilde{\Delta}^2_k\right] \leq \mathbb{E}\left[ \Delta_k\right] \leq \mathbb{E}\left[\tilde{\Delta}^1_k\right].
\end{equation}
\end{lemma}
\begin{proof}
\begin{align}
\mathbb{E}\left[\tilde{\Delta}^1_k\right]-\mathbb{E}\left[\Delta_k\right]&=\mathbb{E}\left[\tilde{\Delta}^1_k\right]-\mathbb{E}\left[\hat{\Delta}_k^1\right] \label{eq:lem4}\\ 
&=\mathbb{E}\left[J^{M_k^1}_{\mu_k,\nu_k} - J^{M_k^1}_{\mu_k,\tilde{\nu}_k}\right] 
\geq0, \label{eq:nasheq}
\end{align}
where we used Lemma \ref{lemma:regreteq} in \eqref{eq:lem4} and the fact that $(\mu_k,\tilde{\nu}_k)$ is a  Nash equilibrium for Markov game $M^1_k$ in \eqref{eq:nasheq}. Using a similar argument, we have $\mathbb{E}\left[\tilde{\Delta}^2_k\right]-\mathbb{E}\left[\Delta_k\right]\leq0.$
\end{proof}
Lemma \ref{lemma:regret2} suggests that we can bound the expected regret by establishing an upper bound on $\sum_{k=1}^{K}\mathbb{E}\left[\tilde{\Delta}^1_k\right]$ and a lower bound on $\sum_{k=1}^{K}\mathbb{E}\left[\tilde{\Delta}^2_k\right]$. To do so, we define the following quantities:
\begin{align}
   N_{t_k}(s, a,b) := \sum_{t=1}^{t_k - 1} \mathbbm{1}_{\{(s_t, a_{t},b_{t}) = (s,a,b)\}}. 
\end{align}
\begin{align} \label{eq:beta}
\beta_k(s, a,b) := \sqrt{\frac{14S \log(2S A B K t_k)}{\max\{1, N_{t_k}(s, a,b)\}}}.
\end{align}
We also introduce a new random variable $\Upsilon$ defined  below:
\begin{align}\label{eq:upsilon}
    \Upsilon:= (2H+4)  \sum_{k=1}^{K} \sum_{h=0}^{H-1} \min\{\beta_k(s_{t_k+h}, a_{t_k+h}, b_{t_k+h}), 1\} \notag\\
    + 4H.
\end{align}
We have the following bounds.
\begin{lemma}\label{lemma:delta_hat_bounds}
\begin{align}
    - \mathbb{E}[\Upsilon] \leq  \sum_{k=1}^{K} \mathbb{E} \left[  \tilde{\Delta}^2_k \right] \leq \sum_{k=1}^{K} \mathbb{E} \left[  {\Delta}_k \right]\leq  \sum_{k=1}^{K} \mathbb{E} \left[  \tilde{\Delta}^1_k \right] \leq \mathbb{E}[\Upsilon].
\end{align}
\end{lemma}
\begin{proof}
 See Appendix \ref{app:A}.    
\end{proof}
\subsection{Proof of Theorem \ref{thm:main}} \label{sec:thm1_proof}
Using  Lemma \ref{lemma:delta_hat_bounds}, we can write
\begin{align}
    \left|\sum_{k=1}^{K}\mathbb{E} \left[  {\Delta}_k \right] \right|\leq \mathbb{E}[\Upsilon].
\end{align}
Since the reward at each time belongs to $[-1,1]$, we also have that $\left|\sum_{k=1}^{K}\mathbb{E} \left[  {\Delta}_k \right] \right|\leq 2KH$. Thus, 
\begin{align}
    \left|\sum_{k=1}^{K}\mathbb{E} \left[  {\Delta}_k \right] \right|\leq \min \left\{ \mathbb{E}[\Upsilon], 2KH \right\}.
\end{align}


\cite[Appendix B]{osband2013more} provided an almost sure upper bound of $\Upsilon$ under any  learning algorithm:
\begin{align}\label{eq:upsilonbound}
\Upsilon\leq 12H^2SAB+12HS\sqrt{7ABKH\log(SABKH)}.
\end{align}
Taking the expectation on $\Upsilon$ and combining it with the worst-case bound, we have the following result:
\begin{align}\label{eq:ciupperbound}
&\min \left\{  \mathbb{E}[\Upsilon] , 2KH \right\} \notag\\
&\leq\min\{12H^2SAB+12HS\sqrt{7ABKH\log(SABKH)}\notag\\
&\qquad\qquad\qquad\qquad\qquad\qquad\qquad\qquad\qquad\qquad,2KH\}\notag\\
&\leq 12HS\sqrt{7ABKH\log(SABKH)}\notag\\
&\qquad\qquad\qquad\qquad\qquad\qquad+\min\{12H^2SAB,2KH\}\notag\\
&\leq 12HS\sqrt{7ABKH\log(SABKH)}+H\sqrt{24KHSAB}\notag\\
&\leq 37HS\sqrt{AB KH \log(S AB KH)},
\end{align}
which implies that $|\mathbb{E}[\mbox{Regret}(K)]|$ is 
\[\mathcal{O}(HS\sqrt{ABKH\log(SABKH)}). \]

\subsection{Arbitrary Opponent/Proof of Theorem \ref{thm:arb_opp}}
We now consider the case where player 1 is using the posterior sampling algorithm (Algorithm \ref{alg:maxalg}) but player 2 is using any arbitrary learning algorithm. Let $\nu_k$ denote the policy used by player 2 in episode $k$. Recall that $(\mu_k, \tilde{\nu}_k)$ is the equilibrium policy pair  generated by Algorithm \ref{alg:maxalg} in episode $k$. We can define player 1's regret using \eqref{eq:regret_def} and \eqref{eq:delta0} as in Section \ref{sec:regret_def}.
As  in Section \ref{sec:analysis}, we define $\hat{\Delta}^1_k$ and $\tilde{\Delta}^1_k$ as:
\begin{equation}
    \hat{\Delta}^1_k:=J^{M_k^1}_{\mu_k,\tilde{\nu}_k}-J^{M^*}_{\mu_k,\nu_k}.
\end{equation}
\begin{equation}
    \tilde{\Delta}^1_k:=J^{M_k^1}_{\mu_k,\nu_k}-J^{M^*}_{\mu_k,\nu_k}.
\end{equation}
We note that the argument used to establish $\mathbb{E}\left[ \Delta_k\right]=\mathbb{E}\left[\hat{\Delta}^1_k\right]$ in Lemma \ref{lemma:regreteq} and to establish $\mathbb{E}\left[ \Delta_k\right] \leq \mathbb{E}\left[\tilde{\Delta}^1_k\right] $ in Lemma \ref{lemma:regret2} rely solely on player 1 using the posterior sampling algorithm.  Therefore, using the same arguments with an arbitrary player 2 gives
\begin{equation}
    \mathbb{E}\left[ \Delta_k\right]=\mathbb{E}\left[\hat{\Delta}^1_k\right] \leq
    \mathbb{E}\left[\tilde{\Delta}^1_k\right]. \label{eq:arb_opp}
\end{equation}
We can now employ the proof of Lemma \ref{lemma:delta_hat_bounds} to conclude that 
\begin{align}
     \sum_{k=1}^{K} \mathbb{E} \left[  {\Delta}_k \right]\leq  \sum_{k=1}^{K} \mathbb{E} \left[  \tilde{\Delta}^1_k \right] \leq \mathbb{E}[\Upsilon],
\end{align}
where $\Upsilon$ is as defined in \eqref{eq:upsilon}.  Repeating the steps in Section \ref{sec:thm1_proof} for the proof of Theorem \ref{thm:main} completes the proof.
\section{Experiments}
\begin{figure}[!htbp]
    \centering
    \includegraphics[width=0.5\linewidth]{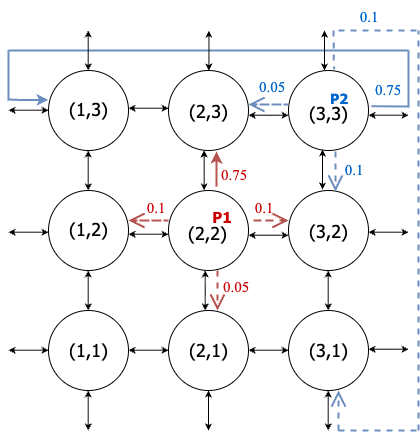}
    \caption{\small{The transition model used in experiments. The red arrows and numbers show the transition probabilities when player 1 at (2,2)  chooses to move upward. The blue arrows and numbers show the transition probabilities when player 2 at (3,3)  chooses to move right. }
    } 
    \label{fig:statediagram}
\end{figure}
\emph{Game Settings:}
 We consider a predator-prey-style two-player zero-sum game for our experiments. Each player stands on a $3 \times 3$ grid, and at each step each player chooses a direction (up, down, left, right) to move. Each player moves one step in its desired direction  with probability 0.75, in the opposite direction with probability 0.05, or in one of the other two directions with probability 0.2 (each direction with probability 0.1, see Figure \ref{fig:statediagram}).   We assume that the grid wraps around at the edges, that is,  if a player goes up from the top row, it will move to the  bottom row in the same column at the next timestep. The transition dynamics for player 1's location are decoupled from player 2's dynamics. We use a 2-dimensional Cartesian coordinates to describe the players' locations.  The reward function of Player 1 is set to be the $\ell_2$ distance between both players times a factor of $1/\sqrt{8}$ for normalization. This implies that Player 1’s objective is to try to maximize its distance from Player 2, and Player 2’s is to minimize it. Note that the reward is deterministic and both players know the reward function before the game starts. Finally, we set the time horizon in each episode to $H=10$, and the distribution of the initial state is the uniform distribution over all possible states. 

\begin{figure}
     \includegraphics[width=0.45\textwidth]{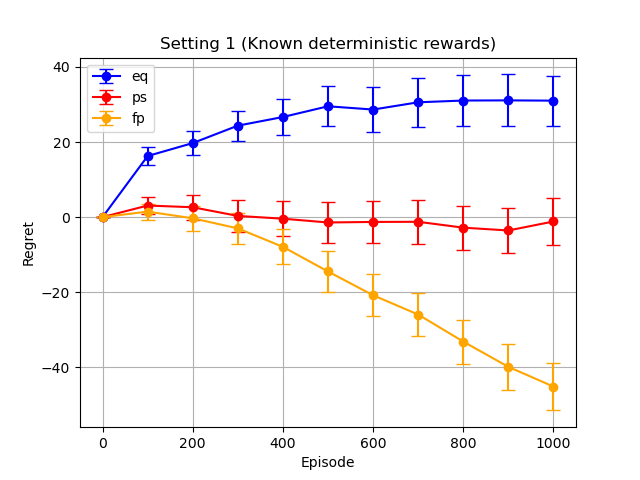}
         \caption{\small{ Player 1's Regret when \textbf{player 1 uses posterior sampling} and player 2 uses (i) true equilibrium strategy (\blue{-- eq}), (ii) player 2 uses fictitious play (\orange{-- fp}), and (iii) player 2 uses posterior sampling (\red{-- ps}). The solid line shows the average of 50 runs and the bar is 95\% confidence interval.}}
         \label{fig:player1ps}
     \end{figure}
     
     \emph{Agent settings:}
Since the players' dynamics are decoupled, the prior distribution on the transition model is the product of   two independent Dirichlet distributions with all  parameters  equal to $1/9$.
In our experiments, we fix player 1 to use the posterior sampling algorithm and consider different learning algorithms for player 2 - a) posterior sampling, b) fictitious play based algorithm (as described below), c) a clairvoyant   algorithm that knows the true game model and  therefore plays the true equilibrium strategy for player 2. 

 The fictitious-play agent operates as follows: Firstly, the agent estimates the game model by the empirical distribution of the state transitions and rewards in the  history. Then it assumes the opponent's strategy is the empirical distribution of opponent's actions in each state. Finally, it calculates its own best response to the estimated opponent's strategy in the estimated model.



\emph{Experiment results:}
Our results   with player 1 using the posterior sampling algorithm are shown in Figure \ref{fig:player1ps}. Each sub-figure shows the average regret of player 1 over 50 runs and the 95\% confidence interval under different algorithms of player 2. Player 1's regret is highest when player 2 is using the true equilibrium strategy. This makes sense since player 2 is better informed in this case (it knows the true game model) and is able to exploit this information superiority. On the other hand, player 1's regret is lowest when player 2 is using fictitious play-based strategy. This suggests that fictitious play based player 2 is not effectively learning the model and is therefore not competing well against a posterior sampling player 1.   When both players use posterior sampling the absolute value of regret remains close to zero, suggesting that the players are somewhat evenly matched.


\section{Conclusions}
In this paper, we studied Bayesian learning in finite-horizon two player zero-sum episodic Markov Games with unknown transition and reward models. We specifically investigated a posterior sampling based learning algorithm where player maintains a posterior distribution over the game model, independently samples a model at the  beginning of each episode, and computes an equilibrium policy for the sampled model.
We established a rigorous theoretical guarantee that shows that the posterior sampling agent achieves sublinear regret on the order of $
O\big(H S \sqrt{A B H K \log(S A B H K)}\big)
$.
Experimental evaluations in a grid-world predator--prey domain illustrate
the sublinear regret scaling and show that posterior sampling
competes favorably with a fictitious-play baseline. Investigating posterior-sampling based learning in non-zero sum games would be an interesting direction for future work.



\appendices

\section{Proof of Lemma \ref{lemma:delta_hat_bounds}}\label{app:A}
Hereafter we will simplify the superscripts as follows:
\begin{equation}
    V^*_{\mu,\nu, h} = V^{M^*}_{\mu,\nu, h}, \quad  V^{i,k}_{\mu,\nu, h}(s) = V^{M^i_k,}_{\mu,\nu,h}(s),
\end{equation}
\begin{equation}
    \mathcal{T}^{*}_{\mu,\nu,h} = \mathcal{T}^{M^*}_{\mu,\nu,h}, \quad \mathcal{T}^{i,k}_{\mu,\nu,h} = \mathcal{T}^{M^i_k}_{\mu,\nu,h}, 
\end{equation}
and 
\begin{equation}
    \quad\overline{R^i_k}=\overline{R}^{M^i_k},\quad\overline{R^*}=\overline{R}^{M^*}.
\end{equation}

The proof uses arguments from Section 5 of \cite{osband2013more}. First consider the conditional expectation of $\tilde{\Delta}_k^1$ conditioned on the true and sampled models. 
\begin{lemma} \label{lemma:new}
For $i=1,2,$
\begin{align}
\mathbb{E} \bigg[  \tilde{\Delta}_k^i\Big| M^*, M^1_k,M^2_k\bigg] =&\mathbb{E} \bigg[ \sum_{h=1}^{H} \Big(\mathcal{T}^{i,k}_{\mu_k,\nu_k,h} - \mathcal{T}^*_{\mu_k,\nu_k,h} \Big)\notag\\
&V^{i,k}_{\mu_k,\nu_k, h+1}(s_{t_k -1 + h}) \Big| M^*, M^1_k, M^2_k \bigg] 
\end{align} 
\end{lemma}
\begin{proof}
    The proof is similar to \cite{osband2013more}. For the sake of completeness, a proof is provided in Appendix \ref{app:B}.
\end{proof}

Using $\beta_k(s,a,b)$, we can define the confidence set for episode $k$:
\begin{align}
\mathcal{M}_k := &\bigg\{ M : \left\| \hat{\theta}_k(\cdot | s,a,b) - \theta(\cdot | s,a,b) \right\|_1 \leq \beta_k(s, a,b), \\ \notag
    & \frac{1}{2}\left|\hat{R}_k(s,a,b)-\overline{R}^M(s,a,b)\right|\leq\beta_k(s, a,b) ~ \forall (s,a,b) \bigg\}
 \end{align}
where $\hat{\theta}_k(\cdot|s,a,b)$ is an empirical distribution 
defined as follows:
\begin{equation}
    \hat{\theta}_k(s'|s,a,b) = \frac{N_{t_k}(s,a,b,s')}{\max\{1,N_{t_k}(s,a,b)\}}
\end{equation}
($N_{t_k}(s,a,b,s')$ is the number of times the tuple $(s,a,b)$ leads to $s'$ in the history $h_{t_k}$), and $\hat{R}_k(s,a,b)$ is the empirical average
reward of the tuple $(s,a,b)$ up to timestep $t_k$. Lemma 17 of \cite{jaksch2010near} shows $\mathbbm{P}(M^* \notin \mathcal{M}_k) \leq 1/K$  for this choice of $\beta_k(s, a,b)$. 
Using this fact along with Lemma \ref{lemma:PS}, we can write
\begin{equation} \label{eq:indicator}
    \mathbb{E}[\mathbbm{1}_{\{M^1_k \notin \mathcal{M}_k\}} ] = \mathbb{E}[\mathbbm{1}_{\{M^* \notin \mathcal{M}_k\}} ] \leq 1/K.
\end{equation}
Since $\tilde{\Delta}^1_k \leq 2H$, we can write:
\begin{align}
    \tilde{\Delta}^1_k \leq &\tilde{\Delta}^1_k \mathbbm{1}_{\{M^1_k, M^* \in \mathcal{M}_k\}}\notag \notag\\
    &  +2H [\mathbbm{1}_{\{M^1_k \notin \mathcal{M}_k\}} + \mathbbm{1}_{\{M^* \notin \mathcal{M}_k\}}]. \label{eq:proof1}
\end{align}
Combining \eqref{eq:proof1} and \eqref{eq:indicator}, we get
\begin{align}
&\sum_{k=1}^{K}\mathbb{E} \left[  \tilde{\Delta}^1_k \right]\leq \sum_{k=1}^{K}\mathbb{E} \left[  \tilde{\Delta}^1_k \mathbbm{1}_{\{M^1_k, M^* \in \mathcal{M}_k\}} \right] + 2H \sum_{k=1}^{K}\frac{2}{K}\notag\\
&\leq \sum_{k=1}^{K}\mathbb{E} \Big[  \mathbb{E} \left[ \tilde{\Delta}_k^1 \big| M^*, M^1_k, M^2_k \right] \mathbbm{1}_{\{M^1_k, M^* \in \mathcal{M}_k\}} \Big] + 4H\notag\\
&= \sum_{k=1}^{K} \sum_{h=1}^{H} \mathbb{E}\big[(\mathcal{T}^{1,k}_{\mu_k,\nu_k,h} - \mathcal{T}^*_{\mu_k,\nu_k,h}) 
 V^{1,k}_{\mu_k,\nu_k, h+1}(s_{t_k+h-1}) \notag\\ 
 &\qquad \qquad \qquad\mathbbm{1}_{\{M^1_k, M^* \in \mathcal{M}_k\}}\big] + 4H \label{eq:proof2}
\end{align}
where we used Lemma \ref{lemma:new} in the last equality above. We can further simplify the right hand side of \eqref{eq:proof2} as
\begin{align}
\leq  &\sum_{k=1}^{K} \sum_{h=0}^{H-1} \mathbb{E}\bigg[\Big(\sum_{s'\in\mathcal{S}}|\theta^1_k(s'|s_{t_k+h}, a_{t_k+h}, b_{t_k+h})\notag\\
&\qquad \qquad-\theta^*(s'|s_{t_k+h}, a_{t_k+h}, b_{t_k+h})|\cdot |V^{1,k}_{\mu_k,\nu_k, h+1}(s')|  \notag\\
& +|\overline{R^1_k}(s_{t_k+h}, a_{t_k+h}, b_{t_k+h})-\overline{R^*}(s_{t_k+h}, a_{t_k+h}, b_{t_k+h})|\Big)\notag\\
& \qquad\qquad\qquad\qquad\qquad\qquad\qquad\mathbbm{1}_{\{M^1_k, M^* \in \mathcal{M}_k\}}\bigg]+4H \notag\\
&\leq   \sum_{k=1}^{K} \sum_{h=0}^{H-1} \mathbb{E}\bigg[\Big(H\cdot\|\theta^1_k(\cdot|s_{t_k+h}, a_{t_k+h}, b_{t_k+h})\notag \\
&\qquad \qquad\qquad \qquad -\theta^*(\cdot|s_{t_k+h}, a_{t_k+h}, b_{t_k+h})\|_1\notag \\
& +|\overline{R^1_k}(s_{t_k+h}, a_{t_k+h}, b_{t_k+h})-\overline{R^*}(s_{t_k+h}, a_{t_k+h}, b_{t_k+h})|\Big)\notag\\
&\qquad\qquad\qquad\qquad\qquad\qquad\qquad\mathbbm{1}_{\{M^1_k, M^* \in \mathcal{M}_k\}}\bigg]+4H \notag
\end{align}
\begin{align}
&\leq   \sum_{k=1}^{K} \sum_{h=0}^{H-1} \mathbb{E}\bigg[\Big(H\big(\|\theta^1_k(\cdot|s_{t_k+h}, a_{t_k+h}, b_{t_k+h})\notag \\
&\qquad \qquad\qquad \qquad -\hat{\theta}_k(\cdot|s_{t_k+h}, a_{t_k+h}, b_{t_k+h})\|_1\notag \\
&  +\|\hat{\theta}_k(s_{t_k+h}, a_{t_k+h}, b_{t_k+h})-\theta^*(s_{t_k+h}, a_{t_k+h}, b_{t_k+h})\|_1\big) \notag\\
&  +|\overline{R^1_k}(s_{t_k+h}, a_{t_k+h}, b_{t_k+h})-\hat{R}_k(s_{t_k+h}, a_{t_k+h}, b_{t_k+h})|\notag\\
&  +|\hat{R}_k(s_{t_k+h}, a_{t_k+h}, b_{t_k+h})-\overline{R^*}(s_{t_k+h}, a_{t_k+h}, b_{t_k+h})|\Big)\notag\\
&\qquad\qquad\qquad\qquad\qquad\qquad\qquad\mathbbm{1}_{\{M^1_k, M^* \in \mathcal{M}_k\}}\bigg]+4H \notag\\
&\leq (2H+4)  \sum_{k=1}^{K} \sum_{h=0}^{H-1} \mathbb{E}\bigg[\min\{\beta_k(s_{t_k+h}, a_{t_k+h}, b_{t_k+h}), 1\}\bigg] \notag\\
&\qquad\qquad\qquad\qquad\qquad\qquad\qquad\qquad\qquad\qquad+ 4H \notag \\
&=\mathbb{E}[\Upsilon].\label{eq:upsilonupperbound}
\end{align}
We can use a similar analysis to bound 
$\sum_{k=1}^{K}\mathbb{E} \left[ -\tilde{\Delta}^2_k \right]$ (note that $ -\tilde{\Delta}^2_k \leq 2H)$:
\begin{align}
 &\sum_{k=1}^{K}\mathbb{E} \left[ -\tilde{\Delta}^2_k \right]\leq \sum_{k=1}^{K}\mathbb{E} \left[ - \tilde{\Delta}^2_k \mathbbm{1}_{\{M^2_k, M^* \in \mathcal{M}_k\}} \right] + 2H \sum_{k=1}^{K}\frac{2}{K} \notag \\
 &\leq \sum_{k=1}^{K}-\mathbb{E} \Big[  \mathbb{E} \left[ \tilde{\Delta}^2_k \big| \theta^*, \theta^1_k, \theta^2_k \right] \mathbbm{1}_{\{M^2_k, M^* \in \mathcal{M}_k\}} \Big] + 4H\notag\\
&= \sum_{k=1}^{K} \sum_{h=1}^{H} -\mathbb{E}\big[(\mathcal{T}^{2,k}_{\mu_k,\nu_k,h} - \mathcal{T}^*_{\mu_k,\nu_k,h}) \notag
 V^{2,k}_{\mu_k,\nu_k, h+1}(s_{t_k+h-1}) \notag\\
  &\qquad \qquad \qquad\mathbbm{1}_{\{M^2_k, M^* \in \mathcal{M}_k\}}\big] + 4H\notag \\
  &\leq  \sum_{k=1}^{K} \sum_{h=0}^{H-1} \mathbb{E}\bigg[\Big(\sum_{s'\in\mathcal{S}}|\theta^2_k(s'|s_{t_k+h}, a_{t_k+h}, b_{t_k+h})\notag\\
&\qquad \qquad \quad-\theta^*(s'|s_{t_k+h}, a_{t_k+h}, b_{t_k+h})|\cdot |V^{2,k}_{\mu_k,\nu_k, h+1}(s')| \notag \\
&\quad+|\overline{R^2_k}(s_{t_k+h}, a_{t_k+h}, b_{t_k+h})-\overline{R^*}(s_{t_k+h}, a_{t_k+h}, b_{t_k+h})|\Big) \notag\\
&\qquad\qquad \qquad \qquad\qquad \qquad \qquad\qquad \qquad\mathbbm{1}_{\{M^2_k, M^* \in \mathcal{M}_k\}}\bigg] \notag\\
 &\leq (2H+4)  \sum_{k=1}^{K} \sum_{h=0}^{H-1} \mathbb{E}\big[\min\{\beta_k(s_{t_k+h}, a_{t_k+h}, b_{t_k+h}), 1\}\big] \notag\\
 &\qquad\qquad\qquad\qquad\qquad\qquad\qquad\qquad\qquad\qquad+ 4H \notag\\
&=\mathbb{E}[\Upsilon].\label{eq:upsilonlowerbound}
\end{align}
By \eqref{eq:upsilonupperbound}, \eqref{eq:upsilonlowerbound}, and Lemma \ref{lemma:regret2}, we have
\begin{align}
- \mathbb{E}[\Upsilon] \leq  \sum_{k=1}^{K} \mathbb{E} \left[  \tilde{\Delta}^2_k \right] \leq \sum_{k=1}^{K} \mathbb{E} \left[  {\Delta}_k \right]\leq  \sum_{k=1}^{K} \mathbb{E} \left[  \tilde{\Delta}^1_k \right] \leq \mathbb{E}[\Upsilon].
\end{align}
These complete the proof.

\section{Proof of Lemma \ref{lemma:new} }\label{app:B}
For $h \in [H]$, let $\rho_{h,k}\in\Delta_{\mathcal{S}}$ be the probability distribution of $s_{t_k -1+h}$ when policies $\mu_k,\nu_k$ are used. Note that $\rho_{1,k}=\rho$ (the initial state distribution). We have the recursive relation for such distributions:
\[
\rho_{h+1,k}(s')=\mathbb{E}_{a\sim\mu_k(s,h),b\sim\nu_k(s,h)}\sum_{s}\rho_{h,k}(s)\theta^*(s'|s,a,b).
\]
Using the Bellman equation, we get:
\begin{equation}
\mathbb{E} \left[ \tilde{\Delta}_k^1 \Big| M^*, M^1_k,M^2_k\right]=\sum_s\rho(s)(V^{1,k}_{\mu_k,\nu_k,1} - V^*_{\mu_k,\nu_k,1})(s) \label{eq:appb.1}
\end{equation}
The right hand side of \eqref{eq:appb.1} can be expanded as
\begin{equation*}
\begin{aligned}
& \sum_s\rho(s)(\mathcal{T}^{1,k}_{\mu_k,\nu_k,1} V^{1,k}_{\mu_k,\nu_k,2} - \mathcal{T}^*_{\mu_k,\nu_k,1} V^*_{\mu_k,\nu_k,2})(s)\\
&\qquad \qquad \quad+\mathcal{T}^{*}_{\mu_k,\nu_k,1} V^{1,k}_{\mu_k,\nu_k,2}- \mathcal{T}^{*}_{\mu_k,\nu_k,1}V^*_{\mu_k,\nu_k,2}\Big)(s) \\
&=\sum_s\rho(s)\Big(\mathcal{T}^{1,k}_{\mu_k,\nu_k,1}V^{1,k}_{\mu_k,\nu_k,2} -\mathcal{T}^{*}_{\mu_k,\nu_k,1}V^{1,k}_{\mu_k,\nu_k,2} \Big)(s) \\
&\qquad \quad+\sum_s\rho(s)\Big(\mathcal{T}^{*}_{\mu_k,\nu_k,1} V^{1,k}_{\mu_k,\nu_k,2}- \mathcal{T}^{*}_{\mu_k,\nu_k,1}V^*_{\mu_k,\nu_k,2}\Big)(s).
\end{aligned}
\end{equation*}
Expanding the second term above, we have:
\begin{equation*}
\begin{aligned}
&\sum_s\rho(s)\mathcal{T}^{*}_{\mu_k,\nu_k,1} (V^{1,k}_{\mu_k,\nu_k,2}(s)- V^*_{\mu_k,\nu_k,2}(s))\\
&=\mathbb{E}_{a\sim\mu_k(s,1),b\sim\nu_k(s,1)}\sum_s\rho(s)\sum_{s'\in\mathcal{S}}\theta^*(s'|s,a,b)\\
&\qquad \qquad\qquad\qquad\qquad\qquad\qquad(V^{1,k}_{\mu_k,\nu_k,2} - V^*_{\mu_k,\nu_k,2})(s')\\
&=\sum_{s'}\left(\mathbb{E}_{a\sim\mu_k(s,1),b\sim\nu_k(s,1)}\sum_s\rho(s)\theta^*(s'|s,a,b)\right)\\
&\qquad\qquad\qquad\qquad\qquad\qquad\qquad(V^{1,k}_{\mu_k,\nu_k,2} - V^*_{\mu_k,\nu_k,2})(s') \\
&=  \sum_{s'}\rho_{2,k}(s') (V^{1,k}_{\mu_k,\nu_k,2} - V^*_{\mu_k,\nu_k,2})(s').\\
\end{aligned}
\end{equation*}
The expression above is similar to \eqref{eq:appb.1} and we can expand it in using similar steps. Doing this recursively, we get 
\begin{equation*}
\begin{aligned}
&\sum_{h=1}^{H}\sum_s\rho_{h,k}(s)(\mathcal{T}^{1,k}_{\mu_k,\nu_k,h} - \mathcal{T}^*_{\mu_k,\nu_k,h} \Big) V^{1,k}_{\mu_k,\nu_k, h+1}(s)\\
&=\sum_{h=1}^{H}\mathbb{E}_{s_{t_k -1+h}\sim\rho_{h,k}}\Big( \mathcal{T}^{1,k}_{\mu_k,\nu_k,h} - \mathcal{T}^*_{\mu_k,\nu_k,h} \Big)\\
 &\qquad \qquad \qquad \qquad \qquad \qquad \qquad V^{1,k}_{\mu_k,\nu_k, h+1}(s_{t_k -1 + h}),
\end{aligned}
\end{equation*}
which is our desired result. The same argument also holds for $\mathbb{E} \left[ \tilde{\Delta}_k^2 \Big| \theta^*, \theta^1_k,\theta^2_k\right]$.
\bibliographystyle{ieeetr}
\bibliography{references_2025}

\end{document}